\documentclass{article}

\usepackage{microtype}
\usepackage{graphicx}
\usepackage{subfigure}
\usepackage{comment}
\usepackage{booktabs} 
\usepackage{graphicx}
\usepackage{dcolumn}
\usepackage{bm}
\usepackage{babel}
\usepackage{tikz-cd} 
\usepackage{amssymb}
\usepackage{mathtools}
\usepackage{lmodern}
\usepackage{amsmath,amssymb}
\usepackage{amsthm}
\theoremstyle{definition}
\newtheorem{definition}{Definition}[section]
\newtheorem{theorem}{Theorem}[section]

\newtheorem{lemma}[theorem]{Lemma}
\newtheorem*{remark}{Remark}
\newtheorem{proposition}{Proposition}[section]

\usepackage{hyperref}

\usepackage[archival]{icml2023}

\usepackage{amsmath}
\usepackage{amssymb}
\usepackage{mathtools}
\usepackage{amsthm}

\usepackage[capitalize,noabbrev]{cleveref}

\usepackage[textsize=tiny]{todonotes}

\icmltitlerunning{A Geometric Insight into Equivariant Message Passing Neural Networks on Riemannian Manifolds}

\begin{document}

\twocolumn[
\icmltitle{A Geometric Insight into Equivariant Message Passing Neural Networks on Riemannian Manifolds}




\begin{icmlauthorlist}
\icmlauthor{Ilyes Batatia}{yyy,xxx}
\end{icmlauthorlist}

\icmlaffiliation{yyy}{ENS Paris-Saclay, Université Paris-Saclay, 91190 Gif-sur-Yvette, France}
\icmlaffiliation{xxx}{University of Cambridge, Cambridge, CB2 1PZ UK}

\icmlcorrespondingauthor{Ilyes Batatia}{ilyes.batatia@ens-paris-saclay.fr}

\icmlkeywords{Machine Learning,}

\vskip 0.3in
]



\printAffiliationsAndNotice{}  
\begin{abstract}
This work proposes a geometric insight into equivariant message passing on Riemannian manifolds. As previously proposed, numerical features on Riemannian manifolds are represented as coordinate-independent feature fields on the manifold. To any coordinate-independent feature field on a manifold comes attached an equivariant embedding of the principal bundle to the space of numerical features. We argue that the metric this embedding induces on the numerical feature space should optimally preserve the principal bundle's original metric. This optimality criterion leads to the minimization of a twisted form of the Polyakov action with respect to the graph of this embedding, yielding an equivariant diffusion process on the associated vector bundle. We obtain a message passing scheme on the manifold by discretizing the diffusion equation flow for a fixed time step. We propose a higher-order equivariant diffusion process equivalent to diffusion on the cartesian product of the base manifold. The discretization of the higher-order diffusion process on a graph yields a new general class of equivariant GNN, generalizing the ACE and MACE formalism to data on Riemannian manifolds.
\end{abstract}

\section{Introduction}

In the last decade, deep learning has emerged as the predominant paradigm for a wide range of applications in machine learning. Exploiting the underlying structure of the data is at the core of the success of deep learning. Convolutional Neural Networks (CNNs) \cite{CNNlecun} exploit the spatial structure of images via the locality of the filters and the weight sharing that enhances its generalizability. 
The success of CNNs in computer vision has made clear the benefits of explicitly using translation equivariant networks. However, many relevant problems exhibit more complex symmetries than images. Such problems benefit from using latent representations related
to the specific underlying symmetry group theory. Equivariant neural networks have emerged as the general class of models that internally transform according to a given symmetry group including permutation invariance for graphs ~\cite{messagepassing, thiede2021autobahn}; spatial rotations for 2D~\cite{CohenSteerable2016, estevespolar2017} and 3D~\cite{Bartok2013OnEnvironments, Anderson2019CormorantCM, ACE_equivariant_ralf, nequip, Batatia2022mace, toshev2023e3} data; or Lorentz boost \cite{Bogatskiy:2022czk, Li:2022xfc, Munoz:2022gjq} or more generally  Lie groups \cite{batatia2023general}. The concept of spaces with symmetries is inherent to all physical phenomena, for instance, the SO(3) group in molecular interaction and the SO(3,1) group in particle physics. Beyond the symmetries, it is crucial to generalize equivariant neural networks to processes on more general domains than the Euclidean space, as physics exhibits many spatial domains, including the Minkowski space in special relativity. Previous study \cite{WeilerManifold2021, cohen2019gauge} generalized CNNs to Riemannian manifolds by defining a notion of coordinate-independent feature fields.

Machine learning on graphs has shown to be successful in a broad range of applications, including physical science \cite{ml_physiscs, bronstein2021geometric}. Large amounts of data are available as graphs; therefore, developing efficient algorithms over graphs is crucial. Graph Neural Networks (GNNs) emerged as one of the most promising methods to apply Deep Learning principles to data over graphs. The relationship between GNNs and diffusion processes has been studied in the attentional flavour \cite{beltrami}. The more general message-passing scheme can be understood as a partial differential equation \cite{messagepassing}. 

This study aims to extend the concept of equivariant message passing to data on Riemannian Manifolds. Our approach demonstrates how, without non-linear activation, one can derive equivariant message passing from the minimization of a specific functional: the twisted Polyakov action. The specialization to Euclidean spaces recovers well known equivariant message passing architectures. This functional can be seen as measuring how accurate the feature field is at embedding the geometry of the manifold along with the action of a given group on it into a vector space.
The update rule of the architecture comes from the discretization of the diffusion process resulting from the minimization of a twisted form of this functional. We propose a higher-order equivariant diffusion process equivalent to diffusion on the cartesian product of the base manifold. The discretization of the higher-order diffusion process on a graph yields a new general class of equivariant GNN closely related to the MACE (multi- Atomic Cluster Expansion) formalism. While the ideas developed in this paper could interest the community, they are far from mature and complete. This paper is intended to spark interest and hopefully lead to a complete understanding of the interplay between modern geometric machine learning architectures and the associated functional they minimize.

In the first chapter, we present Message Passing Neural Networks (MPNNs) formalism and make the equivariance constraint explicit. We explain the connection between MPNNs and diffusion through the discretization of Beltrami flow. In the second chapter, we construct an equivariant feature diffusion to apply the analogy between MPNNs and diffusion processes to equivariant MPNNs. We first represent coordinate independent feature fields on Riemannian manifolds as sections of associated vector bundles. To any coordinate-independent feature field comes a canonically attached equivariant embedding of the principal bundle to the space of numerical features. Minimizing the Polyakov action with respect to the graph of this embedding yields a flow on the associated bundle that is equivalent to a diffusion process. By discretizing this flow for a fixed time step, one obtains a message passing on the manifold. To realize a higher-order message passing on the manifold, we considered Equivariant feature diffusion on the Cartesian product of the manifold allows us to address this question. Finally, we obtain a discrete version of the message passing on the graph by embedding a graph into a Riemannian manifold. We show that such formulation is very closely related to the MACE architecture proposed for the representation of molecules in GNNs in the case of manifolds being $\mathbb{R}^{3}$ and the group $SE(3)$.

\addtocontents{toc}{\vspace{0.5cm}}

\section{Background}

\subsection{Equivariant Message Passing Neural Networks}
Message Passing Neural Networks (MPNN) \cite{messagepassing} are a general class of Graph Neural Networks that parametrize a mapping from labelled graphs to a vector space. In MPNN frameworks, each node is labelled with a state updated via successive message passing between neighbours. After a fixed number of iterations, a readout function maps the state to the space of real numbers. Equivariant MPNNs emerge if the states of the nodes have additional vector features along with an action of a group on them. In this case, one seeks mappings that preserve a given group's action.

\paragraph{Nodes states} Let $\Gamma = (\mathcal{V} = \{1,...,n\},\mathcal{E})$ be a graph with $\mathcal{V}$ and $\mathcal{E}$ denoting the nodes and edges respectively. The state of node $i$, $\sigma_{i}^{(t)}$ is composed of two properties :
\begin{equation}
\label{eq:node_state}
\sigma_{i}^{(t)} = (r_{i},h_{i}^{(t)})
\end{equation}
with $r_{i}$ the positional attribute of the node and $h_{i}^{(t)}$ is a learnable feature of node $i$. These learnable features are 
updated after each iteration of message passing, with an iteration index by $t$.

\paragraph{Message passing and update}

During each round of message passing, features on each node $h_{i}^{(t)}$ are updated based on aggregated messages, $m_{i}^{(t)}$ derived from the states of the atoms in the neighbourhood of $i$, denoted by $\mathcal{N}(i)$ :
\begin{equation}
\label{eq:massage_def}    m^{(t)}_{i} = \bigoplus_{j \in \mathcal{N}(i)} M_{t}(\sigma_{j}^{(t)},\sigma_{i}^{(t)})
\end{equation}
where $\bigoplus_{j \in \mathcal{N}(i)}$ is any permutation invariant operation over the neighbours of node i and $M_{t}$ is a learnable function acting on states of nodes $i$ and $j$. The messages are used to update the features of node $i$ with a learnable update function $U_{t}$:
\begin{equation}
\label{eq:update}
\sigma_{i}^{(t+1)} = (r_{i},h_{i}^{(t+1)}) = (r_{i},U_{t}(m_{i}^{(t)}))
\end{equation}

\paragraph{Readout} After $T$ iterations of message passing and update, in the readout phase, the states of the nodes are mapped to the output $y_{i}$ by a learnable function $R$:
\begin{equation}
y_{i} = R(\sigma_{i}^{(T)})   
\end{equation}

\paragraph{Equivariance}
One can ask for an MPNN to be equivariant with respect to the action of a group $G$. We will consider the case where $G$ is a reductive Lie group, and V is a representation of $G$ with action $\rho$. Formally, a message as a generic function of the states $(\sigma_{i_{1}},...,\sigma_{i_{n}})$ is said to be $(\rho, G)$ equivariant if it respects the constraint:
\begin{equation}
\label{eq:equi-constraint}
m^{(t)}(g \cdot (\sigma_{i_{1}}^{(t)},...,\sigma_{i_{n}}^{(t)})) = \rho(g) \cdot m^{(t)}(\sigma_{i_{1}}^{(t)},...,\sigma_{i_{n}}^{(t)}) \quad \forall g \in G
\end{equation}
where $g \cdot (\sigma_{i_{1}}^{(t)},...,\sigma_{i_{n}}^{(t)})$ denotes an action of $G$ over the states such that 
\begin{equation}
\label{state_action}
g \cdot \sigma_{i}^{(t)} = (g \cdot r_{i},\rho_{h}(g)h_{i}^{(t)} )
\end{equation}
with $\rho_{h}(g)$ denoting the action of g on $h$. 
In practice from equivariance constraint of equation \eqref{eq:equi-constraint} results constraints on the type of operations $M_{t},\bigoplus,U_{t}
,R_{t}$. 

\subsection{Diffusion process and message passing}

Recently a graph Beltrami flow has been introduced  \cite{beltrami} by analogy with diffusion processes on images and related it to a subcase of MPNNs referred to as Attentional graph neural networks. The graph Beltrami flow evolves the features as :
\begin{equation}
\label{beltrami_graph}
\frac{\partial h^{(t)}_{i}}{\partial t} = \sum_{j : i\to j \in \mathcal{E}} a(\sigma_{j}^{(t)},\sigma_{i}^{(t)})(h_{j}^{(t)} - h_{i}^{(t)})
\end{equation}
The solution of this equation for long times minimizes a discretized version of the Polyakov action, which is equivalent to finding an optimal embedding.
With a discretization with a time step of 1, one obtains:
\begin{equation}
h^{(t+1)}_{i} = h^{(t)}_{i} + \sum_{j : i\to j \in \mathcal{E}} a(\sigma_{j}^{(t)},\sigma_{i}^{(t)})(h_{j}^{(t)} - h_{i}^{(t)})
\end{equation}
This update formula corresponds precisely to the attention flavour of MPNNs, and the attention weights $a(\sigma_{j}^{(t)},\sigma_{i}^{(t)})$ can be understood as anisotropic diffusion coefficients.
More general flavours are possible by studying the non-linear equation of evolution of the type $\frac{\partial h^{(t)}_{i}}{\partial t} = \Phi(\{\sigma_{j}\}_{j \in \mathcal{N}(i)})$. In the next section, we will show that equivariant message passing can be achieved as an equivariant diffusion process resulting from a twisted form of the Polyakov action. Other analogies of GNNs with energy minimization have been proposed in \cite{digiovanni2022graph}.

\addtocontents{toc}{\vspace{0.5cm}}

\section{A geometric insight into equivariant message passing}

To apply the rich analogy between diffusion processes and equivariant message passing, one must represent message passing as a diffusion that respects the underlying symmetry imposed by the structure group. For this purpose, we will define an equivariant features diffusion process on a Riemannian Manifold. This will rely on defining equivariant features as sections of an associated bundle as proposed in \cite{cohen2019gauge, WeilerManifold2021}. A general Laplacian will carry the diffusion process that updates the features on the associated bundle. This diffusion process results from minimizing a twisted form of the Polyakov action that preserves equivariance. The equivariant features diffusion process is thus related to finding an optimal mapping between the principal bundle of a manifold and a given irreducible representation of the structure group.

\subsection{Equivarient features fields as sections of associated bundles}

We aim to construct a geometric notion of equivariant features on the manifold as outlined in \cite{WeilerManifold2021, cohen2019gauge}. We will review some fundamentals of differential geometry for this construction. For more mathematical details, see~\cite{zbMATH03194988}. As manifolds do not come with a preferential reference frame (gauge), features must be defined up to an arbitrary choice of coordinates. It is natural to ask for the network to be independent of the coordination as proposed in \cite{WeilerManifold2021}. It results naturally in asking the network to be equivariant under gauge transformations, i.e. local change of reference frame. 
Let ($\mathcal{M},\eta_{\mathcal{M}}$) be a Riemannian manifold of dimension $\mathfrak{m}$.
We will start by making more specific the notion of coordinate independence. To each point $x \in \mathcal{M}$, one can attach a tangent space $T_{x}\mathcal{M}$ which locally looks like $\mathbb{R}^{\mathfrak{m}}$. Each tangent space $T_{x}\mathcal{M}$ is isomorphic to $\mathbb{R}^{\mathfrak{m}}$; however, there is no canonical isomorphism. A preferred choice of local isomorphism is called a gauge \cite{Naber:1997yu}.
\begin{definition}[Gauge]
Let $x \in \mathcal{M}$ and $U^{A}$ be a neighborhood of $x$. A gauge is defined as a smooth and invertible map:
\begin{equation}
    \psi^{A}_{x} : T_{x}\mathcal{M} \to  \mathbb{R}^{\mathfrak{m}} 
\end{equation}
that specifies a preferred choice of isomorphism between $T_{x}\mathcal{M}$ and $\mathbb{R}^{\mathfrak{m}}$.
\end{definition}
Gauges coordinatize tangent spaces only in local neighbourhoods, and in all generality, they can not be extended to the full manifold while preserving their smoothness requirement. One can consider a collection of local neighbourhoods and their corresponding gauges called an atlas   :
\begin{equation}
    \mathcal{A} = \{ (U^{X}, \psi^{X})\}_{X \in \mathcal{X}}
\end{equation}
such that $\bigcup_{X \in \mathcal{X}} U^{X} = \mathcal{M}$. On intersections $U^{A} \cap U^{B} \neq \varnothing $,  the different gauges $\psi^{A}_{x},\psi^{B}_{x}$ are stitched together using smooth transition functions defined on the structure group $G$. From now on, we will consider $G$ to be a reductive Lie group.

\begin{definition}[Transition functions]
A transition function (see~\ref{fig:gauge-maps}) between intersections $U^{A} \cap U^{B} \neq \varnothing $,  is a map :
\begin{equation}
g^{BA} = U^{A} \cap U^{B} \to G , x \mapsto \psi^{B}_{x} \circ (\psi^{A}_{x})^{-1}
\end{equation}
\end{definition}
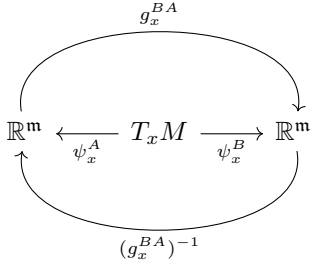
\begin{figure}[h!]
\label{fig:gauge-maps}
\centering
\begin{tikzcd}
\mathbb{R}^{\mathfrak{m}} \arrow[rr,bend left=100,"g^{BA}_{x}"] & \arrow[l,"\psi^{A}_{x} "] T_{x}M \arrow[r,"\psi^{B}_{x}"'] & \mathbb{R}^{\mathfrak{m}} \arrow[ll,bend left=100,"(g^{BA}_{x})^{-1}"]
\end{tikzcd}
\caption{Relationship between gauge maps and transition functions.}
\end{figure}
Therefore one observes that a single coordinate free tangent vector $v \in T_{x}M$ is represented by two vectors in $v^{A},v^{B} \in  \mathbb{R}^{\mathfrak{m}}$ :
\begin{equation}
    v^{A} = g^{AB}v^{B}, g^{BA}\in G
\end{equation}
depending on the gauges $\psi^{A},\psi^{B}$. We see that coordinate independence is closely related to the action of the structure group on the manifold that defines the transition maps between gauges. In practical implementations, we aim for equivariant message passing to diffuse feature fields relative to some local reference frame that take values in some vector space $V$ of dimension $c$. Let $\psi^{A}$ be a local gauge on a neighborhood $U^{A}$ of $M$. Relative to this gauge, a local feature field is defined as :
\begin{equation}
    f^{A} : U^{A} \to V
\end{equation}
given by a vector with $c$-dimensional (corresponding to $c$ channels) feature vector $f^{A}(x)$ at each point $x\in M$. According to some other gauge $\psi^{B}$ on $U^{B}$, one can measure $f^{B}$. As we require the global feature field $h$ to be coordinate independent, the different local feature fields must transform principally. Because gauge transformations are elements of $G$, the local features fields will transform according to some linear representations of $G$ on $V$  :
\begin{equation}
    \rho : G \to Aut(V) = GL(c)
\end{equation}
and we have the local features fields transforming as : 
\begin{equation}
    f^{B}(x) = \rho(g^{BA}_{x})f^{A}(x), x \in U^{A} \cap U^{B}
\end{equation}
The type of representations models different types of features. The scalar fields transform according to the trivial representation $\rho(g) = \mathbf{1}, \forall g \in G$ whose local numerical values are invariant under gauge transformations. More generally, feature fields that transform according to irreducible representations of $G$ are of great practical importance in physics. As $G$ is a reductive Lie group, any finite representation can be decomposed as a sum of irreducible representations. We will therefore restrict ourselves to the case where $V$ is a finite-dimensional representation.

To properly define the diffusion of equivariant feature fields, we must give a geometric interpretation of coordinate-independent feature fields. Global coordinate independent feature fields are sections of an associated vector bundle. Let $G$ be a reductive Lie group and $V$ a finite-dimensional representation of $G$ of dimension c and denote by $\rho$ the linear representation of $G$ on $V$. 
Fiber bundles allow for a global description of fields over a manifold and can thus represent feature fields. 

\begin{definition}[Fiber bundle]
A fiber bundle is a quadruplet $(E,\pi, M, F)$ with E the total space, $M$ the base space and $F$ the typical fiber and a smooth surjective map $\pi : E \to M$. A fiber bundle is locally trivial as $\forall x \in M$ there exists a neighbourhood $U$ of $x$ and a diffeomorphism $\phi : U \times F \to \pi^{-1}(U)$, such that $\pi \circ \phi(x,f) = x, \forall f \in F$. The tuple $(\psi, U)$ is called a local trivialization.
\end{definition}
To include more specific mathematical structures on the typical fiber, one can define subtypes of bundles. An important example is vector bundles, where F is a vector space, and this is crucial as it formalizes the concept of parameterizing a vector space by a manifold.
\begin{definition}[Vector bundle]
A vector bundle is a triplet $(E,\pi,M)$ with $E,M$ two manifolds, and $\pi : E \to M$ such that the preimage $\pi^{-1}(x) \text{ of } x \in M$ has the structure of a vector space.
\end{definition}

To combine geometric properties with differential geometry, one needs to construct a principal bundle $P$, that locally looks like the product of $M$ with a structure group $G$ and such that transition maps are isomorphism \cite{Warner}. 

\begin{definition}[Principal bundle]
A principal bundle is a quadruplet $(P,\pi,M,G)$ where $P$ and $M$ are two manifolds, $G$ a Lie group, $\pi : P \to M$ is a surjective map such that $\pi^{-1}(x)$ is  diffeomorphic to $G$ and there is an action $\cdot$ of $G$ on $P$ such that : 
\begin{itemize}
    \item $\pi(p \cdot g) = \pi(p)$ for $p \in \pi^{-1}(x)$ and $g \in G$
    \item the restriction $G\times \pi^{-1}(x) \to \pi^{-1}(x)$ is free and transitive.
\end{itemize}
We call $M$ the base manifold, $P$ the total space
and $G$ the structure group of the principal bundle. When no confusion can be made, $(P,\pi,M,G)$ is called $P$.
\end{definition}

\noindent The disjoint union of bases of $T_{x}M,x\in M$ that are equivalent by the action of $G$ forms the total space of a principal bundle $PG(TM)$ over $M$ called the bundle of G-frames of $TM$.For now we will refer by $P$ to $PG(TM)$. 

In the following, we will construct associated feature vector bundles with feature coefficients in $V$ as typical fibers. Under gauge transformations, these fibers are acted on by the group linear representation $\rho : G \to Aut(V)$. These features vector bundles are constructed as quotients.

\begin{definition}[Associated bundle]
Let $(P,\pi,M,G)$ be a principal bundle and $\rho$ a linear representation of $G$ on $V$. We define $E = (P \times V) / G$, a point of $E$ is of the form

\begin{equation}
\label{Associated_bundle_point}
    [p,h] = \{(p\cdot g, \rho(g^{-1})h),g \in G\}
\end{equation}

where $p\in P$ and $h \in V$. Let $\pi_{E}: E$ be given by $\pi_{E}[p,h] = \pi(p)$. Then $(E,\pi_{E},M)$ forms a vector bundle called an associated vector bundle to P, denoted $P\times_{(\rho,G)}V$.
\end{definition}

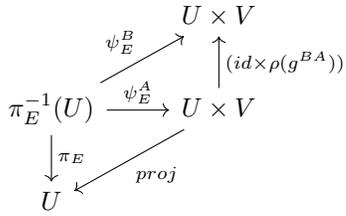
\begin{figure}[h!]
\begin{minipage}[b]{\linewidth}
   \centering
    \begin{tikzcd}
     & U \times V & \\
    \pi^{-1}_{E}(U) \arrow[r,"\psi_{E}^{A}"] \arrow[ur,"\psi_{E}^{B}"] \arrow[d,"\pi_{E}"] & U\times V \arrow[u,"(id \times \rho(g^{BA}))"'] \arrow[dl,"proj"]\\
    U 
    \end{tikzcd} 
\caption{Trivialization of $E \xrightarrow[]{\pi_{E}} M$}
\end{minipage}
\end{figure}

The construction of $E$ is equivalent to coordinate independent feature vectors on M : features $f(x) \in E$ are expressed relative to arbitrary frames in $P$.

\begin{definition}[Coordinate free features field]
Coordinate free features fields are defined as smooth global sections $f \in \Gamma(E)$ that is a smooth map $f: M \to E$ such that $\pi_{E} \circ f = id_{M}$. 
\end{definition}

\begin{lemma}
The coordinate free feature fields $f \in \Gamma(E)$ and ($\rho,G$)-equivariant functions $h$ on $P$ are in one-to-one correspondence such that one can attach to any $h$ a canonical coordinate free feature field $f_{h}$.
\end{lemma}

\begin{proof}
Define a ($\rho,G$)-equivariant function $h : P \to V$ from the principal bundle to the representation $V$. Let $p_{1},p_{2} \in \pi^{-1}(x)$ and define $f_{h}^{1}(x) = [p_1,h(p_1)]$ and $f_{h}^{2}(x) = [p_2,h(p_2)]$. As  $p_{1},p_{2} \in \pi^{-1}(x)$, there exists a $g\in G$ such that $p_{1} = p_{2} \cdot g$. By $(\rho,G)$-equivariance of $h$ and by the definition of a point of the associated bundle \eqref{Associated_bundle_point},
\begin{align}
&f_h^{1}(x) = [p_1,h(p_1)] = [p_2 \cdot g ,h(p_2 \cdot g)] = 
\\ &[p_2\cdot g,\rho(g^{-1})h(p_{2})] = [p_1,h(p_1)]
\end{align} we have $f_{h}(x)$ invariant by the choice of $p\in \pi^{-1}(x)$. Thus for any $p \in \pi^{-1}(x)$, $\pi_{E} \circ f_{h}(x) = x$, and $f_{h}(x) \in \Gamma(E)$. Conversely, for $f \in \Gamma(E)$, we put $h_{f}(p)= v$ such that $f \circ \pi(p) = [p,v]$. Then $h_{f}$ is $(\rho,G)$-equivariant. Therefore the coordinate-free features fields naturally generate a corresponding equivariant map. We will denote $f_{h}$ the coordinate-free feature map canonically associated to the equivariant function $h$.
\end{proof}

\begin{definition}[Canonical association function]
 We will write $\gamma$ the map associating to an equivariant function $h \in  C^{\infty}(P,V)^{(\rho,G)}$ the section of $E$, $f_{h}$ such that $f_{h} = [p,h(p)]$ :
\begin{equation}
    \gamma : C^{\infty}(P,V)^{(\rho,G)} \to \Gamma(E), h \mapsto f_{h}
\end{equation}
\end{definition}
On local neighborhood $U^{A}$ of $M$, the coordinate-free feature field is trivializable on an arbitrary local frame by the action of a gauge $\psi^{A}_{E,x}$ into numerical features $f_{h}^{A}$: 

\begin{equation}
    f_{h}^{A}(x) = \psi^{A}_{E,x}(f_{h}(x)), x \in U^{A}
\end{equation}

A different choice of trivialization on different neighbourhoods $U^{B}$  yields different numerical features related by the action of the linear representation of $V$

\begin{equation}
    f^{B}_{h}(x) = \rho(g^{BA})\psi^{A}_{E,x}(f_{h}(x))
\end{equation}



In all generality, the features consist of multiple independent feature fields on the same base space transforming according to different representations. The whole space is defined as the Whitney sum $\bigoplus_{i}E_{i}$. A common practice in equivariant message passing the construction of feature fields from a set of finite dimensional irreducible representations with the highest weight $\lambda$.

\subsection{Diffusion of features fields and Polyakov Action}

This section considers a coordinate-free feature field $f_{h} \in \Gamma(E)$ introduced in the previous section over a single finite-dimensional representation $V$ of $G$. We argue that an implicit regularization for the feature field is to realize an optimal embedding of $M$ in $E$, finding the optimal featurization of $M$ into the representation $V$ that respects the underlying structure imposed by the group $G$. This optimality constraint applies to its corresponding equivariant function $h$. In order to make the notion of optimality more precise, we will introduce the Polyakov action that measures the energy of an embedding and will be our optimality criterion.

We will assume that $G$ is a simply connected compact Lie group with Lie algebra $\mathfrak{g}$. For locally compact reductive Lie groups, most of the following discussion can be applied identically by considering the exposition on the maximal compact subgroup of the universal cover of the complexification of $G$ as they share the same finite-dimensional representations. One can induce a Riemannian metric $u$ on $P$ using the the Riemannian metric on $\mathcal{M}$, $\eta_{\mathcal{M}}$ and the invariant metric on $G$, $\eta_{G}$. To do so, consider the tangent space of $P$ at the point $p \in P$, $T_{p}P$. It can be decomposed into a sum of two subspaces, $V_{p}P$ called the vertical space, which is the kernel of the pushforward $\pi^{*}: T_{p}P \to T_{\pi(p)M}$ and the horizontal space $H_{p}P$ which is the complementary subspace. One can define naturally a connection on $P$ by looking at the following map $\phi_{p} : \mathfrak{g} \to V_{p}P$:
\begin{equation}
    \phi_{p}(X) = \frac{d}{dt}(p \cdot \exp(tX) )|_{t=0}, \quad X \in \mathfrak{g}
\end{equation}
In the case where $G$ is a compact simply connected Lie groups, this map is invertible, and we call $\phi_{p}^{-1}$ the inverse. A similar map can be formed if the group has a finite number of disconnected components by multiplying the map by a discrete group $\mathbf{H} \in G$ containing a representative from each connected component. From this map, one can define an inner product on $P$. Let $dp_{1} = v_{1} + h_{1}$ and $dp_{2} = v_{2} + h_{2}$ be two points on $T_{p}P$, the inner product is defined as,
\begin{equation}
(dp_{1}, dp_{2}) = \eta_{\mathcal{M}}(\pi(h_{1}), \pi(h_{2})) + \eta_{G}(\phi_{p}^{-1}(v_{1}), \phi_{p}^{-1}(v_{2}))
\end{equation}
The metric $u$ on $P$ is the metric induced by this inner product.

Moreover, we assume that $P\times V$ is equipped with a Riemannian metric $v = u \oplus \kappa \mathbf{I}_{d}$, for an arbitrary positive number $\kappa$ multiplying the identity of $V$. 
The graph of $h$ given by $\varphi_{h} : P \to P \times V$ realizes an embedding of $P$ into $P\times V$. This embedding induces a natural metric on $P \times V$ given by $\gamma_{\mu, \nu} = \frac{\partial \varphi_{h}^{i}}{\partial x_{\mu}}\frac{\partial \varphi_{h}^{j}}{\partial x_{\nu}} v_{i,j}$. We will now make precise our optimality criterion.

\begin{definition}[Optimality of the feature fields]
A feature field $f_{h} \in \Gamma(E)$ is optimal if the induced metric associated with the embedding $\varphi_{h}$ of the principal bundle to the numerical features space $V$ preserves optimally the metric on the Principal bundle. 
\end{definition}

A natural way to think about this optimal condition is that one wishes to construct features in a vector space that preserves as much as possible the geometry of the initial manifold.
To measure the energy of $\varphi_{h}$ and correlatively of $f_{h}$, we study its Polyakov Action.

\begin{definition}[Polyakov action]
The Polyakov action of the embedding $\varphi$ is defined as,
\begin{equation}
    S[\varphi_{h},u,v] = \int_{P} u^{\mu,\nu} \frac{\partial \varphi_{h}^{i}}{\partial x_{\mu}}\frac{\partial \varphi_{h}^{j}}{\partial x_{\nu}} v_{i,j} dP
\end{equation}
Minimizing the Polyakov with respect to embedding metrics $u$ and $v$ is equivalent to finding the optimal embedding  $\varphi_{h,opt}$ of $P$ in $P \times V$.
\end{definition}
The minimization of the precedent action yields the resolution of the Euler-Lagrange equation. One obtains a gradient descent flow in the form of a heat equation on $h : P \to V$ :
\begin{equation}
    \frac{\partial h_{t}}{\partial t} = \Delta^{P} h 
\end{equation}
where $ \Delta^{P}$ corresponds to the Laplace Beltrami operator on the principle bundle $P$. However, this flow equation is not guaranteed to preserve equivariance. To preserve equivariance of the initial condition $h^{(0)}$, \cite{batard} proposed introducing a twisted version of the Polyakov action resulting in a gradient descent flow. The additional term in the action is expressed as a scalar product $\langle , \rangle$ on $V$.

\begin{definition}[Casimir Operator]

Let $\mathfrak{g}$ be the Lie algebra of $G$. Let $(\mathfrak{g}_{1},...,\mathfrak{g}_{n})$ be an orthonormal basis of $\mathfrak{g}$ and $d\rho : \mathfrak{g} \to GL(V)$ the representation of $\mathfrak{g}$ induced by $\rho$. Let $(\mathfrak{g}^{1},...,\mathfrak{g}^{n})$ be the dual basis on $\mathfrak{g}$ with respect to the Killing form. The Casimir operator $Cas \in GL(V)$ is defined as :
\begin{equation}
    Cas = \sum_{i}^{n} d\rho(\mathfrak{g}^{i})d\rho(\mathfrak{g}_{i})
\end{equation}
\end{definition}
The twisted Polyakov action is given by adding a term to the Polyakov action :
\begin{equation}
        S[\varphi_{h},u,v] = \int_{P} u^{\mu,\nu} \frac{\partial \varphi_{h}^{i}}{\partial x_{\mu}}\frac{\partial \varphi_{h}^{i}}{\partial x_{\nu}} v_{i,j}  + \frac{1}{2} \langle \text{Cas} \cdot h , h\rangle dP
\end{equation}
Minimizing the equation with respect to $\varphi_{h}$, one obtains the following heat equation :
\begin{equation}
    \frac{\partial h^{t}}{\partial t} = -(\Delta^{P} \otimes \mathbb{I} + \mathbb{I} \otimes Cas) h^{t}
\end{equation}
which one can rewrite :
\begin{equation}
    \frac{\partial h^{t}}{\partial t} = \Delta^{E} h^{t}
\end{equation}
where we define $\Delta^{E} = - (\Delta^{P} \otimes \mathbb{I} + \mathbb{I} \otimes Cas)$ the generalized Laplacian on the associated bundle $E$. In the special case where $V$ is an irreducible representation, $Cas = \lambda \mathbb{I}$ and $\Delta^{E} =  - (\Delta^{P} \otimes \mathbb{I} + \lambda \mathbb{I})$.

\begin{definition}[Equivariant Feature Diffusion Process]

Let $h^{(0)} : P \to V$ be an initial equivariant feature map. Let $\Delta^{E}$ be the generalized Laplacian on the associated bundle $E$. We define an Equivariant Diffusion Process as follows:
\begin{equation}
\label{eq:GM-message-passing}
 h^{(T)} = h^{(0)} + \int_{0}^{T} \Delta^{E}h^{(t)}dt \quad y = \mathbf{R}(h^{(T)})
\end{equation}
The feature map $h^{(T)}$ will be the optimal equivariant map from $P$ to the irreducible representation $V$ of $G$ for $T$ sufficiently long. Furthermore, $\mathbf{R}$ is a learnable readout function. 
\end{definition}

From the equivariant diffusion process, we can construct a smooth coordinate independent feature field using the feature field canonically attached to any equivariant function on $P$ :
\begin{equation}
    f^{(t)}_{h}(x) = [p,h^{(t)}(p)], p \in \pi^{-1}(x)
\end{equation}

\subsection{Equivariant features propagator on reductive groups}

The diffusion process on the associated bundle, $E$, can be understood in terms of the scalar heat kernel on $P$ as detailed in \cite{Berline1992HeatKA}. First, we observe the following remark :
\begin{remark}
Let $\Delta^{E} = - (\Delta^{P}\otimes \mathbb{I} + \mathbb{I}\otimes Cas)$, then for some parameter $t$, $$e^{-t\Delta^{E}} = e^{-tCas}e^{-t\Delta^{P}}$$
\end{remark}
We have the following equality for $p_{1}\in P$
\begin{equation}
    (e^{-t\Delta^{E}} h)(p_{1}) = \int_{E} \langle p_{1} | e^{-t\Delta^{E}}|p_{2}\rangle h(p_2) dp_{2}
\end{equation}
where $\langle p_{1} | e^{-t\Delta^{E}}|p_{2}\rangle $ is the Schwartz Kernel of $ e^{-t\Delta^{E}}$ in the Dirac notation.

\begin{proposition}[Getzler - Vergne - Berline]
If $p_{1}$ is a point of the principle bundle $P$  of a locally compact group $G$ and $\Delta^{E}$ the generalized Laplacian over the associated bundle $E$. 
Set $p_{2}$ is a chosen representative in $\pi^{-1}(x)$, for $x \in \mathcal{M}$. Then for any function $h : P \to V$, and $|dx|$ a Riemannian density,
\begingroup\makeatletter\def\f@size{8.5}\check@mathfonts
\def\maketag@@@#1{\hbox{\m@th\large\normalfont#1}}%
\begin{equation}
(e^{-t\Delta^{E}} h)(p_{1}) = e^{-tCas}\int_{G}\int_{\mathcal{M}} \langle p_{1} | e^{-t\Delta^{P}}|p_{2}g\rangle \rho(g)^{-1}h(p_{2}) dg dx
\end{equation}\endgroup
In the case of $G$ is a non-compact reductive Lie group, the first integral taken over the maximal compact subgroup of the complexification of $G$ referred to as $K_{\mathbb{C}}$, following the Weyl unitary trick.
\end{proposition}

\begin{definition}
Let's $e^{-t\Delta^{E}}$ be the equivariant propagator of the diffusion such that : 
\begin{equation}
\label{eq:propagator}
h^{(t')} = e^{-(t'-t)\Delta^{E}}h^{(t)}
\end{equation}
\end{definition}

\begin{definition}
For any $p_{1},p_{2} \in P$, denote 
\begin{equation}
\label{eq:heat-kernel}
    \langle p_{1} | e^{-(t'-t)\Delta^{P}} | p_{2} \rangle = k_{t'-t}(p_{1},p_{2})
\end{equation}
the heat kernel of $\Delta^{P}$ such that one can rewrite the propagator of equation \eqref{eq:propagator} as 
\begin{align}
    &h^{(t')} = e^{-(t'-t)\Delta^{E}}h^{(t)} = \\ &e^{-tCas}\int_{\mathcal{M}} \int_{G} k_{t'-t}(p_{1},p_{2}g)\rho(g)^{-1}h^{(t)}(p_{2}) dg dx
\end{align}

\end{definition}



\begin{definition}[GM - Message passing]

Let $h^{(0)} : P \to V$ be an initial equivariant feature map. Let $\Delta^{E}$ be the generalized Laplacian on the associated bundle $E$. We define the message operation in a GM - Message passing as :
\begingroup\makeatletter\def\f@size{8.5}\check@mathfonts
\def\maketag@@@#1{\hbox{\m@th\large\normalfont#1}}%
\begin{equation}
\label{eq:GM-message-passing}
    m^{(T)}(p_{1}) = e^{-TCas}\int_{\mathcal{M}} \int_{G} k_{T}(p_{1},p_{2}g)\rho(g)^{-1}h^{(0)}(p_{2}) dg dx 
\end{equation}\endgroup
\begin{equation}
    h^{(T)}(p_{1}) = U^{(T)}(m^{(T)}(p_{1}))
\end{equation}
with the function $U^{(T)}$ a learnable equivariant function. The usual equivariant updates are linear combinations from the vector space of $C^{\infty}(P, V)^{(\rho, G)}$.
\end{definition}

\begin{remark}
  GM-Message passing networks are gauge equivariant. A weight-sharing constraint needs to be preserved for the method to be also equivariant to the diffeomorphism of the manifold. GM-Message passing networks are indeed diffeomorphism-preserving maps as the heat kernel is shared across the manifold.
\end{remark}

As this diffusion process has low correlation order (only pairwise interaction), it can be inefficient at modeling highly correlated interactions. Therefore, it is convenient to construct a higher-order GM-Message passing on a Cartesian product of manifolds. First, define $P^{n}$ the principal bundle on the base manifold $M^{n}$.

\begin{definition}[Higher order GM - Message passing]

Let $\Tilde{h}^{(0)} : P^{n} \to V^{n} $, be an initial equivariant feature map. Let $E^{n}$ be the associated vector bundle to $P^{n}$ and $\Delta^{E^{n}}$ be the generalized Laplacian on the associated bundle $E^{n}$. Let $dx^{n} = \prod_{\xi}^{n} dx_{\xi}$ be a volume element on the manifold $\mathcal{M}^{n}$. We define a Higher order GM - Message passing as :
\begingroup\makeatletter\def\f@size{8.5}\check@mathfonts
\def\maketag@@@#1{\hbox{\m@th\large\normalfont#1}}%
\begin{align}
\label{eq:HG-GM-message-passing}
    &m^{(T)}(p_{1}) = \\ 
    &e^{-TCas}\int_{\mathcal{M}^{n}} \int_{G} \Tilde{k}_{T}^{N}(p_{1},p_{2}g,..,p_{n}g)\rho(g)^{-1}\Tilde{h}^{(0)}(p_{2},...,p_{n}) dg dx^{n} \notag
\end{align}\endgroup
\end{definition}
Assume the following structure on the diffusion kernel and feature function,
\begingroup\makeatletter\def\f@size{8.5}\check@mathfonts
\def\maketag@@@#1{\hbox{\m@th\large\normalfont#1}}%
\begin{align}
\label{structure_assumption1}
\Tilde{k}_{T}(p_{1},p_{2}g,..,p_{n}g) = \prod^{n}_{\xi=1}k_{T}(p_{1},p_{\xi}g) \\
\label{structure_assumption2}
\Tilde{h}^{(0)}(p_{2}g,...,p_{n}g) = \rho(g)^{n-1} \prod_{\xi=1}^{n} h^{(0)}(p_{\xi}g)
\end{align}
\endgroup
We can rewrite the Higher-order GM-Message Passing of equation \eqref{eq:HG-GM-message-passing} as :
\begin{align}
\label{eq:HG-GM-message-passing_struct}
    &m^{(T)}(p_{1}) = \\ &e^{-TCas} \int_{G} \rho(g)^{-1} dg \prod^{n}_{\xi = 1}\int_{\mathcal{M}} k_{T}(p_{1},p_{\xi}g)h^{(0)}(p_{\xi}) dx_{\xi}\notag
\end{align}
\begin{equation}
    h^{(T)}(p_{1}) = U^{(T)}(m^{(T)}(p_{1}))
\end{equation}
The propagated equivariant map is projected back to numerical features using the canonically attached feature field and the gauge,
\begin{equation}
    f^{A,(T)}_{h}(x) = \psi^{A}_{E} \circ \gamma(h^{(T)}(p)), p \in \pi^{-1}(x)
\end{equation}

\subsection{Propagation beyond Diffusion}

Equivariant diffusion allows for an insightful interpretation regarding the minimization of the Polyakov action. It also gives theorems and techniques to understand geometrically the diffusion kernel. 
However, diffusion poses severe limitations.
The update function should be linear to conserve the original heat equation, which is not always valid in message-passing neural networks. However, this corresponds to the case of MACE \cite{Batatia2022mace}, which uses linear updates.
We allow GM - Message passing in all generality to have a non-linear update function. The nonlinearity should be of a gate form to avoid breaking equivariance \cite{WeilerManifold2021}. It is for future work to analyze the impact of nonlinearities update on the process. 

\subsection{Equivariant Message Passing over graphs }

In the previous section, we introduced a method to update the coordinate-free features field by applying an equivariant diffusion process to the canonically attached equivariant map. In this section, we will apply this formalism to discretized manifolds.
We consider now $\mathcal{G} = (\mathcal{V} = \{1,...,n\},\mathcal{E})$ a graph with $\mathcal{V}$ and $\mathcal{E}$ denoting nodes and edges respectively. Let ($\mathcal{M},\eta_{\mathcal{M}}$) be a Riemannian manifold of dimension $\mathfrak{m}$. The induced topology on the graph by the metric $\eta_{\mathcal{M}}$ is defined by the set $\mathcal{E} = \{(i,j) \in \mathcal{V}  | \eta_{\mathcal{M}}(r_{i},r_{j}) \leq r_{c}\}$ of neighboring points. By fixing one point $i$, one can construct from the set $\mathcal{E}$, neighbors of $i$ defined as $\mathcal{N}(i) = \{j \in \mathcal{V}  | \eta_{\mathcal{M}}(r_{i},r_{j}) \leq r_{c} , j \neq i \}$. We find the discrete analogy to the GM-Message passing, equivalent to finding the optimal equivariant map $h : P \to V$ with information on a discretized manifold represented by the graph $\mathcal{G}$.

\begin{definition}[State function]
We define $(P,\pi,\mathcal{M},G)$ the principal bundle of $\mathcal{M}$. Let $\sigma :\mathcal{V} \to \mathcal{M} \times  V $ be the state function of the nodes (atoms) $\mathcal{V}$  such that $\sigma_{i} = (r_{i},f_{h,i})$ with $r_{i}$ the positions and $f_{h,i}$ the canonically attached coordinate independent feature field attached to $h : P \to  V $. We call $\sigma^{A}$ the state relative to a gauge $\psi^{A}$ on $\mathcal{M}$.
\end{definition} 
\begin{figure}[h]
    \centering
    \begin{tikzcd}[column sep = 4em, row sep=3em, crossing over clearance = 1em]
    & & U \times V \\
    \mathcal{V} \arrow[urr,"\sigma^{A}"] \arrow[r,"r"] \arrow[rr,"f_{h,i}",bend right, near end ] & U  \arrow[r,"f_{h}"] \arrow[d] &  \pi^{-1}_{E}(U) \arrow[u,"\psi_{E}^{A}"'] \\
    & \pi^{-1}_{P}(U) \arrow[r,"h"]& V
    \end{tikzcd}
    \caption{The state function labels the set of nodes $\mathcal{V}$ by a tuple of positional features in $\mathcal{M}$ and a the numerical feature field in $V$ a vector space.}
\end{figure}
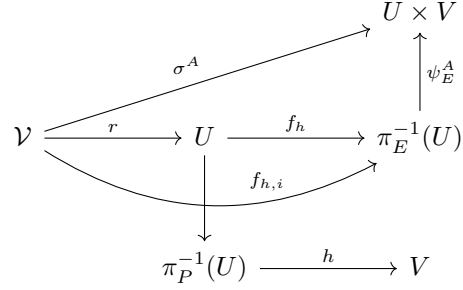
\vspace{-6pt}
\begin{definition}[Equivariant Message Passing]
Let $h^{(0)} : P \to V$ be an initial equivariant feature map. Let $\Delta^{E}$ be the generalized Laplacian on the associated bundle $E$ and $k_{T}$ the heat kernel of $\Delta^{E}$. We define the message operation of Equivariant - Message passing on $\mathcal{G}$ :
\begingroup\makeatletter\def\f@size{8.5}\check@mathfonts
\def\maketag@@@#1{\hbox{\m@th\large\normalfont#1}}%
\begin{equation}
\label{eq:Equivariant-message-passing}
    m^{(t)}(p_{i}) = e^{-tCas} \int_{G} \rho(g)^{-1}dg \prod^{n}_{\xi = 1} \sum_{j \in \mathcal{N}(i)} k_{t}(p_{i},p_{j}g)h^{(t-1)}(p_{j})
\end{equation} \endgroup
The feature map $h^{(T)}$ will be the optimal equivariant map from $P$ to the irreducible representation $V$ of $G$ for $T$ sufficiently long over the graph $\mathcal{G}$.
We identify this operation as the pooling operation of message passing,
\vspace{-5pt}
\begingroup\makeatletter\def\f@size{8.5}\check@mathfonts
\def\maketag@@@#1{\hbox{\m@th\large\normalfont#1}}%
\begin{align}
&\bigoplus_{j \in \mathcal{N}(i)} M_{t} (\sigma_{j}^{(t)},\sigma_{i}^{(t)}) = \notag
\\ &e^{-tCas} \int_{G} \rho(g)^{-1}  dg \prod^{n}_{\xi = 1} \sum_{j \in \mathcal{N}(i)} k_{t}(p_{i},p_{j}g)h^{(t-1)}(p_{j})
\end{align}
\end{definition} \endgroup
The kernel function can be learnable in all generality depending on the iteration $t$, such as a neural network. The update function $U_{t}$ can be a linear or non-linear neural network, shifting from classical diffusion to a partial differential equation evolution.

Let's assume that, $\mathcal{M} = \mathbb{R}^{3}$, $G = SE(3)$ and $V$ is the irreducible representation of invariant scalar to $SE(3)$, such that $\rho(g) = I, \forall g \in G$. If one expands the heat kernel into a spherical series, we recover the equation of the MACE architecture \cite{Batatia2022mace, Batatia2022de}.

Let $k(p,q)$ be a translationally invariant kernel. Then $\forall t \in \mathbb{R}, k(p + t, q + t) = k(p,q) = k(0,q-p) = \Tilde{k}(q-p)$. Thus we see that $\Tilde{k}$ and $k$ contain as much information.

\paragraph{Spherical expension of the kernel}First consider the heat kernel on the Sphere $S^{2}$. By Mercer's theorem, any kernel $k(p - q) \in L(S^{2})$ can be expended in the form of spherical harmonics.
By extending it to $\mathbb{R}^{3}$ any kernel on $k(p-q) \in L({R}^{3})$ can be expended into,
\begingroup\makeatletter\def\f@size{8.5}\check@mathfonts
\def\maketag@@@#1{\hbox{\m@th\large\normalfont#1}}%
\begin{equation}
    k(p - q) = \sum_{n=0}^{+\infty}\sum_{l=0}^{+\infty}\sum_{m=-l}^{m=l}R^{(n)}(\boldsymbol{p} - \boldsymbol{q})c_{lm} Y_{m}^{l}(\hat{p} - \hat{q})
\end{equation} \endgroup
where $Y_{m}^{l}$ are spherical harmonics of order $lm$
By truncating the expansion to a maximal $l$ value and injecting it into the previous message passing equation, one recovers the exact equation for the MACE messages if the time step is constant as the Casimir term will just become a constant re-scaling.

\vspace{-5pt}
\section{Discussion}
In this work, we have introduced a new geometric interpretation of message passing based on differential geometry and diffusion that led to the formulation of a general class of networks on Riemannian manifolds. Implementing these models on test manifolds is still needed, and numerical experiments are required to validate the proposed approach. Fast Fourier transform on the group could simplify the computation of the integral over the group arising in the formulation of $GM$-message. The connection with equivariant message passing could also be fruitful for the interpretability of these methods and to help create new numerical schemes based on our proposed understanding.   
A more general discussion is needed on the connection between message passing and non-linear partial differential equations on manifolds. Extending this work beyond point clouds using non-commutative geometry would represent a challenging task but might be a fruitful endeavour.

\section{Acknowledgements}
The author would like to express heartfelt gratitude towards the reviewers, denoted as $uudT$ and $8mvL$. The invaluable feedback and constructive critiques provided have been instrumental in shaping and refining this work.

\nocite{langley00}

\bibliography{example_paper}
\bibliographystyle{icml2023}

\end{document}